\documentclass[letterpaper, 10 pt, journal, twoside]{IEEEtran}

\IEEEoverridecommandlockouts                              

\usepackage{amsthm}
\newtheorem{theorem}{Theorem}

\newtheorem{assumption}{Assumption}
\newtheorem{remark}{Remark}
\newtheorem{definition}{Definition}

\usepackage{algorithm}
\usepackage{algpseudocode}
\usepackage{amsmath,amssymb,amsfonts}
\usepackage{bm}
\usepackage{booktabs}
\usepackage{cite}
\usepackage{color}
\usepackage{comment}
\usepackage{fancyhdr}
\usepackage{float}
\usepackage{gensymb}
\usepackage{graphicx}
\usepackage[hidelinks]{hyperref} 
\usepackage{mathrsfs}
\usepackage{soul}
\usepackage{textcomp}
\usepackage{times}
\usepackage[dvipsnames]{xcolor} 
\usepackage{url}

\DeclareMathOperator*{\argmax}{argmax}

\newcommand{\boldtheta}{\boldsymbol{\theta}}
\newcommand{\R}{\mathbb{R}}

{\vspace{-\topsep}\begin{itemize}\itemsep1pt \parskip0pt \parsep1pt}
{\end{itemize}\vspace{-\topsep}}

{\vspace{-\topsep}\begin{enumerate}\itemsep1pt \parskip0pt \parsep1pt}
{\end{enumerate}\vspace{-\topsep}}

\begin{document}

\title{Design Optimization for Large High-Force Soft Robot Manipulators Under Gravitational Loads}


\markboth{}
{Cholaseuk \MakeLowercase{\textit{et al.}}: Design Optimization for Soft Manipulators} 

\author{Isara Cholaseuk$^{1}$, 
Penelope Llibre$^{1}$, 
Alexa Kyriacou$^{1}$, 
Audrey Wang$^{1}$, 
Akua K. Dickson$^{2}$, \\
Ran Jing$^{1}$, 
Juan C. Pacheco Garcia$^{1}$, 
Andrew P. Sabelhaus$^{1,2}$
\thanks{$^{1}$I. Cholaseuk, P. Llibre, A. Kyriacou, A. Wang, R. Jing, J.C. Pacheco Garcia, and A.P. Sabelhaus are with the Department of Mechanical Engineering, Boston University, Boston MA, USA
        {\tt\footnotesize \{isara, pllibre, alexakyr, axwang, rjing, jcp29, asabelha\}@bu.edu}}%
\thanks{$^{2} $A.K. Dickson and A.P. Sabelhaus are with the Division of Systems Engineering, Boston University, Boston MA, USA
        {\tt\footnotesize akuad@bu.edu}}%
\thanks{Digital Object Identifier (DOI): see top of this page.}
}

\maketitle

\pagenumbering{arabic}

\begin{abstract}

Designing large soft robots capable of generating high forces for physical human-robot interaction remains a significant challenge in soft robotics. 
Prior work in large soft robots has focused on proof-of-concept prototypes, and no systematic framework exists for determining the suitability of a design paradigm for a desired task. 
This manuscript introduces a method for optimizing the geometry of a soft robot limb, maximizing its blocking force subject to an anti-bucking constraint under its own gravitational loading. 
We demonstrate that an explicit solution exists to the proposed optimization problem under certain assumptions.
Experiments with three geometries of a large, soft, pneumatically-actuated manipulator demonstrate that the method correctly predicts which designs meet constraints and which produces the largest end-effector forces.
This method, with its closed-form solution, can allow designers to determine a-priori if an intended class of soft manipulators is an appropriate choice for physical interaction at large size scales.

\end{abstract}


\section{Introduction}
\label{sec:intro}

\IEEEPARstart{H}{ow} large can soft robots be?
Though there is much interest in soft robot manipulators for human-scale tasks \cite{chen_review_2022} such as healthcare assistance \cite{manti2016soft,cianchetti2018biomedical,ilievski2011soft,fathi_deployable_2019}, no formal methodology exists to determine their appropriateness at larger sizes.
State-of-the-art commonly uses softness only at the end effector \cite{shintake_soft_2018,turco_grasp_2021}, or incorporates rigid or stiff materials \cite{sanders_dynamically_2023,jensen_tractable_2022,truby_recipe_2021,lipton_modular_2019}, or deploys under loading conditions that do not require the robot to support its own weight \cite{guan2020novel,marchese2016design,peng2019dimension,Calisti_Poseidrone_2015, Cianchetti_OCTOPUS_2015}.
Other large soft manipulators are either bulky \cite{johnson2021using,good2025torque,Li_modular_2022,gillespie_simultaneous_2016,best_control_2015} or move very slowly \cite{he2019electrically}.
Is it realistic to develop robots using only soft materials that meet our engineering goals at the human scale?

This paper addresses the question by a method for optimizing the design of a soft manipulator.
The approach considers choosing some desired geometric dimension to meet constraints on buckling under gravity, while also maximizing applied forces at the robot's end effector.
We demonstrate that, under certain conditions and assumptions, an analytical solution exists for the optimal geometry.
Experiments with three large robot prototypes (Fig. \ref{fig:teaser}) demonstrate that our method correctly predicts both buckling behavior and optimal-force designs.
The approach could allow practitioners to determine if a fully-soft manipulator is a viable choice for a human-scale robotic interaction problem.

\begin{figure}[!t]
    \centering
    \includegraphics[width=1\linewidth]{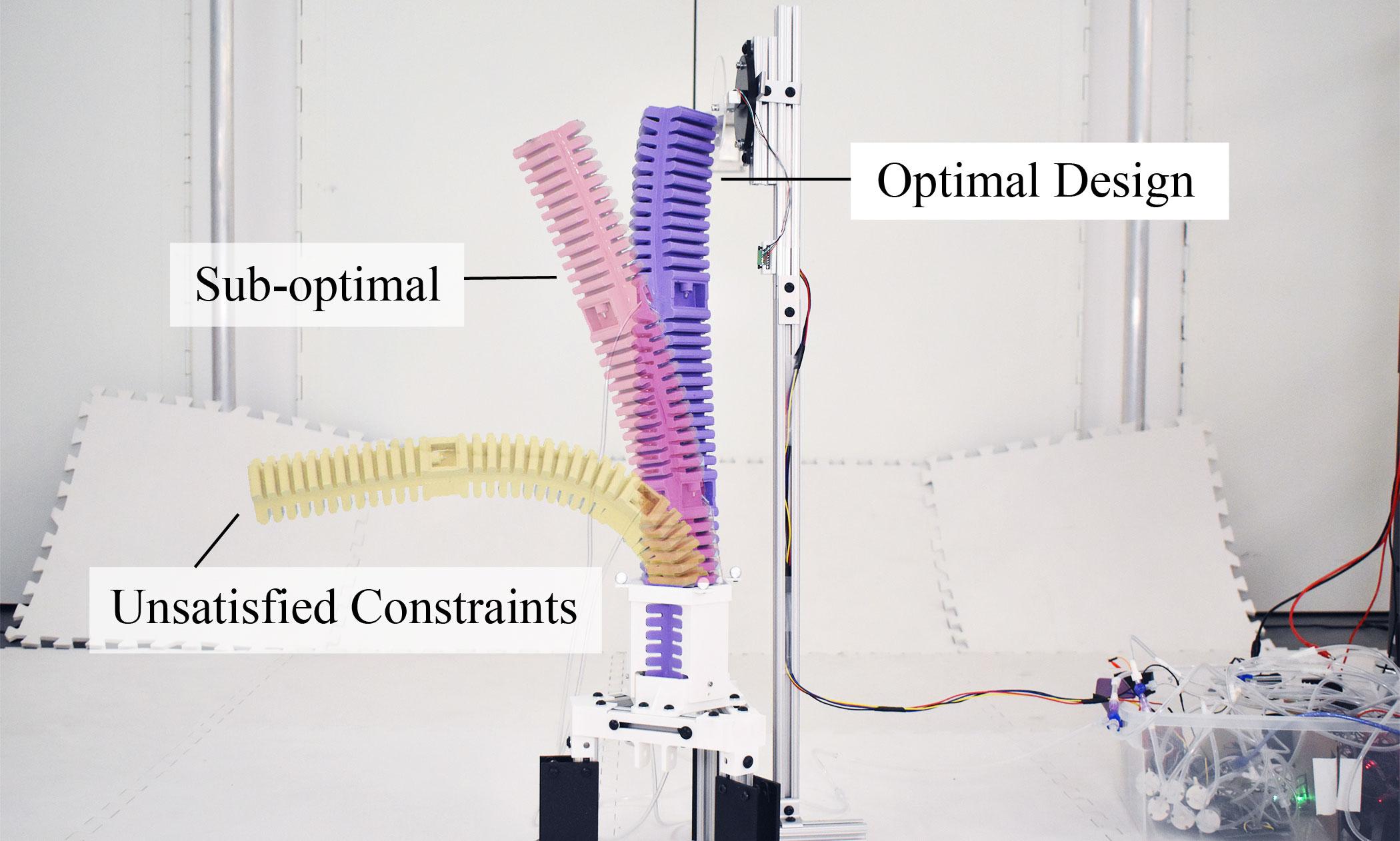}
    \caption{We propose an optimization framework for the geometry of a soft limb that maximizes the blocking force subject to a fixed amount of material and anti-buckling conditions. Hardware tests validate the anti-buckling constraint (yellow) as well as optimality for blocking force (pink vs. purple).}
    \vspace{-0.3cm}
    \label{fig:teaser}
\end{figure}

Prior work in soft robot geometry optimization does not allow this one-shot determination.
Some work relies on finite element analysis (FEA) to iteratively change design parameters until a local optimum is reached \cite{Moseley2015,Xavier2021}, commonly optimizing for workspace \cite{peng2019dimension} or range-of-motion \cite{hu2018structural} or against ballooning instability \cite{elsayed2014finite}.
This is highly task-specific and robot-specific, limiting generalizable conclusions at early stages of the design process.
Other optimization methods for soft robots use rod models, and focus on locomotion objectives rather than manipulation and forces \cite{schaff_sim--real_2023,armanini_model-based_2022}, or do not come with analytical optimality proofs \cite{pinskier_bioinspiration_2022}.
To our knowledge, no past work has optimized a soft limb under gravity without the need for iteration \cite{Chen2020, della_santina_control_survey}. 

Our method overcomes this challenge through careful approximations, a precise problem specification, and prior work in mechanics.
We propose that blocking force is an informative surrogate for more complicated physical interaction objectives, and that gravitational self-buckling \cite{greenhill} is the primary failure mode to prevent.
Under certain conditions -- specifically, with a fixed mass of the robot and a monotonic model of force -- an analytical solution exists to our constrained optimization.
We then apply the framework to a pneumatically-actuated (PneuNet) soft limb, chosen for its uniform material properties and fast response time \cite{mosadegh2014pneumatic} as well as known blocking force \cite{zhu2021modeling}.
Experiments show that our optimal geometry satisfies the constraints and outputs the highest force.

\subsection{Contribution}

This paper seeks to answer: how can soft robot limbs be optimally designed for interaction forces output while supporting their own weight?
In doing so, we contribute:
\begin{enumerate}
    \item a methodology for optimizing the geometric design of soft robotic limbs under gravity-induced loading,
    \item a heuristic approach to approximate maximum stable heights of beams with complex geometries, and
    \item a validation of the proposed approach through hardware experiments.
\end{enumerate}

Though our approach is applied to PneuNet manipulators, the problem is formulated independently of actuation mechanism.
Therefore, the core concept in this manuscript may be generalizable to other styles of soft manipulators.

\section{Methodology}
\label{sec:method}

We formulate the design of a soft robotic limb as an optimization problem that captures the tradeoff between maximizing blocking force output and maintaining structural stability under gravitational loading.

\subsection{Problem Setup}
\label{subsec:prob_setup}

Consider an arbitrary bending soft robotic limb made of homogeneous material (Fig. \ref{fig:problem_setup}(a)) defined by the design parameter vector $\boldsymbol{\theta}=[\theta_1,\theta_2,...,\theta_{n}]^\top \in \mathbb{R}^n$, each of which represents an independent geometrical dimension. Here, $n$ is the number of design parameters. We also consider $L$, which corresponds to the overall length of the limb, but do not directly optimize it. In practice, these geometric parameters are chosen based on different architectures of soft limbs.

\begin{figure}[h]
    \centering
    \vspace{0.2cm}
    \includegraphics[width=1\linewidth]{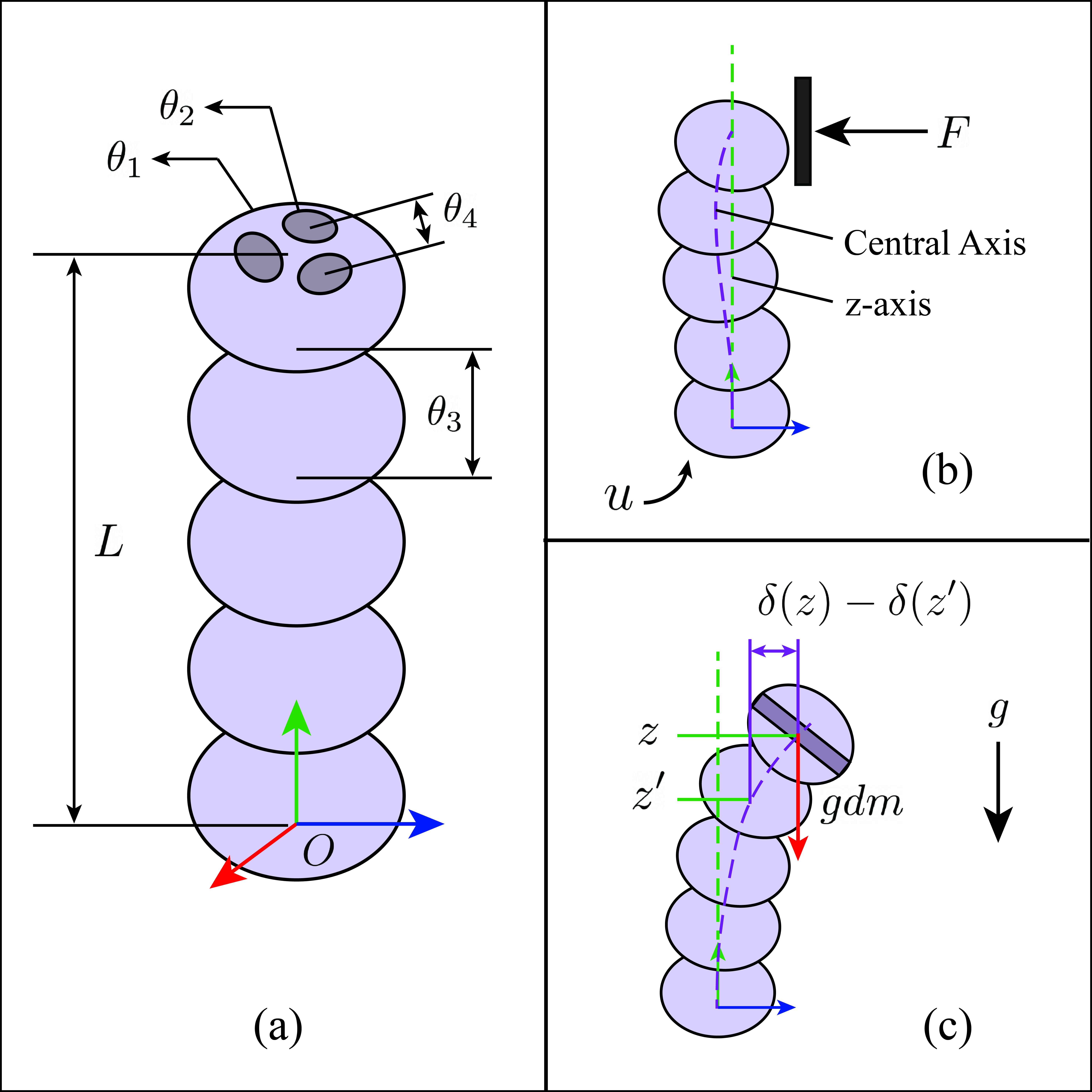}
    \caption{Problem setup of an arbitrary soft robotic limb with complex geometrical shape. a) Arbitrary geometrical design parameters of the soft limb. b) Actuated limb interacting with external blocking force. c) Limb's distributed weight under gravity.} 
    \label{fig:problem_setup}
    \vspace{-0.2cm}
\end{figure}

To derive relationships between $\boldtheta$ and the blocking force $F$, we consider interaction forces when the limb is in contact with an object (Fig. \ref{fig:problem_setup}(b)) at the static equilibrium. An actuation input $u$, such as pneumatic pressure, is applied to the limb.
At each point $z$ along the backbone, the limb's central axis deflects from the $z$-axis by $\delta(z)$.
The blocking condition imposes $\delta(L)=0$.
We assume that, in the blocking position, $\delta(z) \approx 0 \ \forall z$, so the distributed weight along the central axis exerts negligible moment around $O$. This static equilibrium is used to model the relationship between $F$, $u$, and $\boldtheta$.

Additionally, to investigate the condition required for the limb not to buckle under its own weight, we consider the case when the limb is passively standing (Fig. \ref{fig:problem_setup}(c)). 
At the static equilibrium, there exists two equally opposing moments: the gravitational loading along the beam's backbone and the moment due to the internal stiffness of the limb at any deflection $\delta(z)>0$.
We assume a uniform, linearly elastic beam oriented vertically with fixed and free ends, having a constant cross-sectional area of $A$, a constant second moment of inertia of $I$ along the arbitrary length of $L$. 
The robot has a fixed density of $\rho$ and a fixed Young's modulus of $E$. At the static equilibrium, each point $z$ along the beam's central axis is deflected by $\delta(z)$ along a radial direction. 
Consequently, at each point $z'$, there exists a moment of magnitude $\rho g A\int_{z'}^L \delta(z) - \delta(z') dz $ due to cumulative weight above it, which is opposed by the internal stiffness of the beam $-EI\frac{d^2\delta(z)}{dz^2}$ (Fig. \ref{fig:problem_setup}(c), uniform limb). 

It is known from Euler-Bernoulli beam theory \cite{greenhill} that there exists a maximum stable height $L_{max}$ for the beam under these conditions. 
The equilibrium point becomes unstable when $L>L_{max}$, inducing self-buckling.
Greenhill's solution \cite{greenhill} to this problem is the following, where we note that $I$ and $A$ are functions of the geometric dimensions $\boldtheta$:

\begin{equation}
    \label{eq:greenhill}
    L_{max}(\boldtheta) = \left( 7.8373 \dfrac{EI(\boldtheta)}{\rho g A(\boldtheta)} \right)^{1/3}.
\end{equation}

\noindent 
Since $L_{max}(\boldtheta): \R^n \mapsto \R$, the ratio $I/A$ determines $L_{max}$. 

\subsection{Optimization Approach}

We propose to find $\boldtheta^*$ that maximizes the blocking force by assuming $F$ to be a function of the limb's geometry and the actuator input, $F(\boldtheta,u): \R^{n}\times \R \mapsto \R$. 
Since we aim to find the optimal design parameters $\boldtheta^*$ that maximizes $F$, we consider a fixed value of the input $u=\bar{u}$ so that $F(\boldtheta): \R^n \mapsto \R := F(\boldtheta,\bar{u})$. 

There are many ways to pose the remaining problem setup; we envision a simplified case for initial design investigations of a class of robot geometries.
We propose a maximum mass of robot, which maps to a maximum volume of material: $V(\boldtheta): \R^n \mapsto \R$, $V< \bar{V}$.
Additionally, there must be some minimum length of the limb for workspace reasons; we set the length $L=\bar{L}$.
The second constraint then becomes anti-buckling, $L_{max}(\theta) \geq \bar{L}$.
Combined, the optimization problem is:

\begin{align}
\boldsymbol{\theta}^* &= \operatorname*{argmax}_{\boldsymbol{\theta}} F(\boldsymbol{\theta}) \label{eq:obj} \\
\text{s.t.}\hspace{0.3cm} 
& V(\boldsymbol{\theta}) \leq \bar{V} \label{eq:volume_const} \\
& L_{max}(\boldsymbol{\theta}) \geq \bar{L}  \label{eq:gh_const}
\end{align}

\noindent To solve $\boldtheta^*$ analytically, we introduce Assumptions \ref{assm:closed_bounded} and \ref{assm:f_increase}, followed by Theorem \ref{thm:active}. Let a set of feasible designs $\Omega \in \R^n$ be a set of all design variables that satisfies \eqref{eq:volume_const} and \eqref{eq:gh_const}, as:
\begin{equation}
\label{eq:omega}
    \Omega = \{ \boldtheta \in \R^n : V(\boldtheta) \leq \bar{V}, L_{max}(\boldtheta) \geq \bar{L} \}
\end{equation}

\begin{assumption}[Convex $\Omega$]
    \label{assm:closed_bounded} 
    For any $\boldtheta_1$, $\boldtheta_2 \in \Omega$, the set $\{ \gamma \boldtheta_1 + (1-\gamma) \boldtheta_2 \} \in \Omega$ for any $0 \leq \gamma \leq 1$, i.e., $\Omega$ is convex.
\end{assumption}

\begin{assumption}[Componentwise monotonicity]
\label{assm:f_increase} 
We assume that $F(\boldtheta)$ is smooth and differentiable on $\boldtheta \in \Omega$, that $\left| \frac{\partial F}{\partial \theta_k} \right| > 0$ for all $\theta_k \in \boldtheta \in \Omega$ and $k\in\{1,2,...,n\}$, i.e., $F(\boldtheta)$ is strictly monotonic in each coordinate.
\end{assumption}

Our insight is that there exists a unique, optimal solution $\boldtheta^*\in \Omega$, solved analytically under Assumptions \ref{assm:closed_bounded} and \ref{assm:f_increase}:

\begin{theorem}[Active Constraint]
\label{thm:active}
    Consider the objective function $F(\boldtheta)$ with $\boldtheta \in \Omega$. If Assumptions \ref{assm:closed_bounded} and \ref{assm:f_increase} hold, the maximizer \eqref{eq:obj} is globally optimal and lies on the boundary of $\Omega$, where at least one inequality from \eqref{eq:volume_const} and \eqref{eq:gh_const} is active.
\end{theorem}

\begin{proof}
    Let $F(\boldtheta)$ be smooth on the convex set $\Omega \subset \R^n$ (Assumption \ref{assm:closed_bounded}), and suppose $F$ is either strictly increasing ($ \frac{\partial F}{\partial \theta_i}  > 0$) or strictly decreasing ($ \frac{\partial F}{\partial \theta_j} < 0$) for all $i,j \in \{1,2,...,n\}$ where $i \neq j$ (Assumption \ref{assm:f_increase}).\\
    \indent Let $\boldtheta^*$ be a global maximizer of $F$ over $\Omega$, i.e., $\boldtheta^* \in \argmax_{\boldtheta \in \Omega} F(\boldtheta)$.
    We will prove by contradiction that $\boldtheta^*$ cannot lie in the interior of $\Omega$; hence $\boldtheta^*$ must lie on the boundary
    of $\Omega$, i.e., at least one constraint must be active.\\
    \indent Assume that $\boldtheta^*$ is an interior point of $\Omega$, for contradiction. Since $\Omega$ is convex, there exists an open neighborhood $U \subset \Omega$ that contains $\boldtheta^*$.
    Because $F$ is smooth and strictly increasing in direction $\boldsymbol{e}_i$, there exists sufficiently small $\epsilon_1>0$ s.t. $\boldtheta^* + \epsilon_1 \boldsymbol{e}_i \in U \subset \Omega$, and by strict positivity of partial derivative in that direction, $F(\boldtheta^* + \epsilon_1 \boldsymbol{e}_i) > F(\boldtheta^*)$.\\
    \indent Similarly, because $F$ is strictly decreasing in direction $\boldsymbol{e}_j$, there exists sufficiently small $\epsilon_2 > 0$ s.t. $\boldtheta^* - \epsilon_2 \boldsymbol{e}_j \in U \subset \Omega$, and by strict negativity of partial derivative in that direction, $F(\boldtheta^* - \epsilon_2 \boldsymbol{e}_j) > F(\boldtheta^*)$. The existence of $F(\boldtheta^* + \epsilon_1 \boldsymbol{e}_i)$ and $F(\boldtheta^* - \epsilon_2 \boldsymbol{e}_j)$ contradict the assumption that $\boldtheta^*$ lies in the interior of $\Omega$. Therefore, $\boldtheta^*$ must lie on the boundary of $\Omega$, i.e., at least one of the inequality constraint \eqref{eq:volume_const} or \eqref{eq:gh_const} is active. Specifically, either $V(\boldtheta^*)=\bar{V}$ or $L^* = L_{max}$, where $L^* \in \boldtheta^*$.
\end{proof}
\begin{remark}
\label{rem:thm}
    The primary use of Theorem \ref{thm:active} is in choosing the appropriate constraints during the design and manufacturing process, $\bar{V}$ and $\bar{L}$, that make Assumptions \ref{assm:closed_bounded} and \ref{assm:f_increase} true, which leads to an analytical solution to the optimization problem.
\end{remark}

\begin{remark}
    \label{rem:thm_limitation}
    We note that the global optimizer $\boldtheta^*$ is not necessarily unique, depending on the boundary of $\Omega$ and the objective function $F$. These details do not appear in our proof-of-concept below, and so are left for future work.
\end{remark}

\subsection{Heuristic Approximation for Nonuniform Geometries}

There exists a challenge in applying Eq. \eqref{eq:greenhill} in constraint \eqref{eq:gh_const} to some classes of soft limbs.
Eq. \eqref{eq:greenhill} assumes the limb to be cross-sectionally uniform, meaning the properties $I(\boldtheta)$ and $A(\boldtheta)$ are constant along the limb's length, which is not always true, such as for the cartoon limb shown in Fig. \ref{fig:problem_setup}a. We propose that if Assumption \ref{assm:sym_periodic} holds, Assumption \ref{assm:uniform_beam} can be applied to our limb. 

\begin{assumption}
    \label{assm:sym_periodic}
    $A(\boldtheta,z)$ and $I(\boldtheta,z)$ are symmetrical and periodic along the central axis $z$.
\end{assumption}

\begin{assumption}
    \label{assm:uniform_beam}
    If Assumption \ref{assm:sym_periodic} holds, then $L_{max}(\boldtheta)$ can be under-approximated by that of a uniform beam with $I_{min} = \min_z I(\boldtheta,z)$ and $A_{approx} =  V(\boldtheta)/L$.
\end{assumption}

Applying Assumption \ref{assm:uniform_beam}, we modify Eq. \eqref{eq:greenhill} by approximating $I=I_{min}$ and $A=A_{approx}$. Considering the denominator $\rho A(\boldtheta) = \rho A_{approx} = \rho V(\boldtheta)/L$, which is the limb's average linear density, denoted as $\lambda(\boldtheta) = \rho V(\boldtheta)/L$. 
Inserting these assumptions produces a modified $L_{max}$ to be used as constraint (\ref{eq:gh_const}),

\begin{equation}
    \label{eq:approx_greenhill}
    L_{max}(\boldtheta) = \left( 7.8373 \dfrac{EI_{min}(\boldtheta)}{ g \lambda(\boldtheta)} \right)^{1/3}
\end{equation}

We will show in section \ref{sec:app:obj} that our limb design satisfies the conditions of Assumption \ref{assm:sym_periodic}, and in section \ref{sec:res:buck} that Assumption \ref{assm:uniform_beam} holds, i.e., Eq. \eqref{eq:approx_greenhill} can approximate $L_{max}(\boldtheta)$.

\begin{remark}
\label{rem:lmax}
    An analytical derivation of $L_{max}(\boldtheta)$ of a non-uniform beam is beyond the scope of this manuscript and is left for future work to evaluate against Assumption \ref{assm:uniform_beam}.
\end{remark}

\section{Application to Pneumatic Network Actuators}
\label{sec:application}

We apply the optimization methodology to an example architecture to demonstrate its validity.
We choose to design our soft limb based on the fast PneuNet \cite{mosadegh2014pneumatic} (Fig. \ref{fig:limb_dimension}(a)) due to its quick inflation time and known blocking force model.
The proof-of-concept design attaches three individual PneuNets (Fig. \ref{fig:limb_dimension}(b)) to create the simplest possible three-dimensional robot arm.
The actuator input $u$ is pneumatic pressure.

\begin{figure}[h]
    \centering
    \includegraphics[width=1\linewidth]{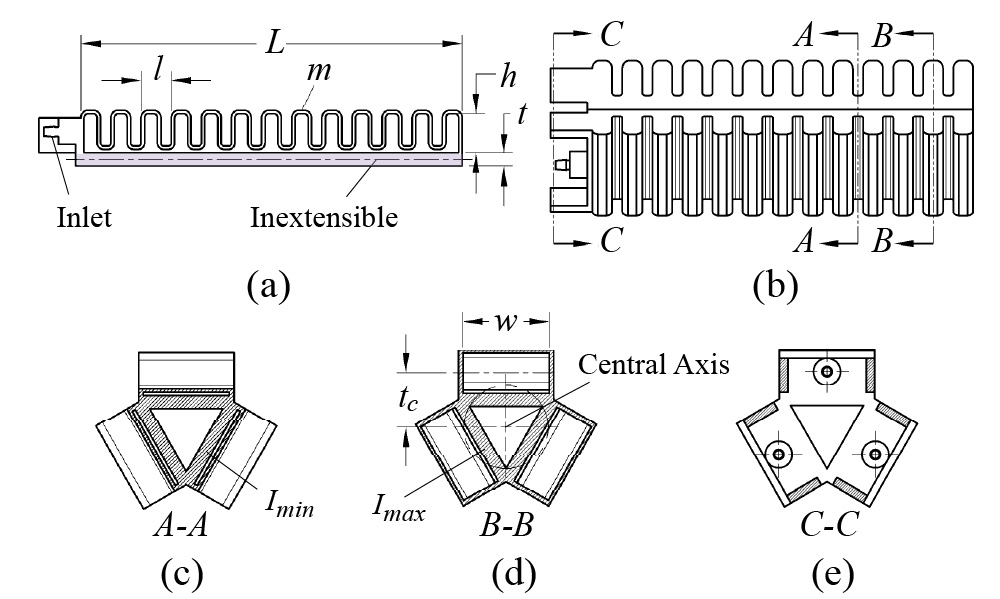}
    \caption{Design parameters of the prototype PneuNet soft manipulator. (a) 1D PneuNet. (b) Prototype limb made of three PneuNets, with the bottom side oriented on the left. (c) Cross-section at $I_{min}$ in hatched area. (d) Cross-section at $I_{max}$ in hatched area. (e) Cross-section at the connection in hatched area.}
    \label{fig:limb_dimension}
\end{figure}

\subsection{Design Parameters}
 
Consider a one-dimensional PneuNet with multiple chambers (Fig. \ref{fig:limb_dimension}(a)). 
We first parameterize our limb geometry (Fig. \ref{fig:limb_dimension}(b)) by $\boldtheta := [h, w, t, m, \ell, L]^\top \in \R^6$, where $h$ is the chamber height, $w$ is the chamber width, $t$ is the the inextensible layer thickness, $m$ is the number of chambers, $\ell$ is the length of a single chamber, and $L$ is the total limb length (See Fig. \ref{fig:limb_dimension}(a) and \ref{fig:limb_dimension}(d)). 

Cross-sectional views A–A and B–B in Fig. \ref{fig:limb_dimension}(c) and \ref{fig:limb_dimension}(d) depict the cross sections with minimum and maximum second moment of inertia $I_{min}$ and $I_{max}$, respectively. Cross-sectional view C-C 
(Fig \ref{fig:limb_dimension}(e)) shows the cross-section of the connecting area, which either connects to a base fixture or concatenates multiple segments of limb (further discussed in Section \ref{sec:ex:fab}). 

For proof-of-concept, we fix additional geometric quantities due to manufacturing limits (omitted). The chamber length $\ell$ relates to the total height and number of chambers by $\ell:= L/m$. We set $\ell$ to be constant (Table \ref{tab:variables}), and because we predetermine that $L=\bar{L}$, $m$ is also a constant. 
For a proof-of-concept, we fix $h$ and $t$ (Table \ref{tab:variables}) and allow only $w$ to vary. 
This reduces manufacturing variation.
Hence, $\boldtheta:= w \in \R$.

\begin{table}[h]
    \centering
    \caption{Numerical Values Used for Optimization Analysis}
    \begin{tabular}{c c c c}
    \hline \vspace{-0.2cm} \\
        \vspace{0.08cm}
        Quantities & Variables & Values & Units \\
        \hline \vspace{-0.2cm} \\
        Fixed Variables & $l$ & $0.014$ & m \\
        & $h$ & $0.018$ & m \\
        & $t$ & $0.006$ & m \vspace{0.08cm} \\
        \hline \vspace{-0.2cm} \\
        Design Requirements & $\bar{L}$ & $0.49$ & m \\
        & $\bar{V}$ & $8.387e^{-4}$ & m$^3$ \vspace{0.08cm} \\
        \hline \vspace{-0.2cm}  \\
        Material Properties \cite{smooth_sil} & $E_{sil}$ & $1.793$ & MPa \\
        & $\rho_{sil}$ & $1240$ & kg/m$^3$ \vspace{0.08cm}\\
        \hline
    \end{tabular}
    \label{tab:variables}
\end{table}

The robot is constructed from Smooth-On Smooth-Sil 945 \cite{smooth_sil}, giving $E$ and $\rho$.
Here, $\bar{L}$ is selected to be roughly close to the length of a human's arm, and $\bar{V}$ is selected based on $1.04 \ kg$ of Smooth-Sil 945.

\subsection{Modeling of PneuNet Blocking Force}
\label{sec:application-force}

To determine the objective function for our optimization problem, we consider the blocking force model of a 1D PneuNet. Much prior work has developed analytical blocking force models of the 1D PneuNet, such as \cite{wang2022analytical}, \cite{alici2018modeling}, and \cite{zhu2021modeling}.
For the scope of this manuscript, we consider $F$ derived by \cite{zhu2021modeling} due to the similarity of geometry, simplified to:

\begin{equation}
\label{eq:force}
    F(w) = 2(\frac{7}{4}-\alpha) \bar{u} \frac{(m-1) hwt_c}{L}.
\end{equation}

Here, $\alpha$ represents a coefficient term, which can be found experimentally \cite{zhu2021modeling}. $t_c$ is the distance between the neutral axis (location at which there is no longitudinal strain) and the location of the effective pressure that causes the moment (See Fig. \ref{fig:limb_dimension}(d)). The actuator input $u$ in this case represents pneumatic gauge pressure.

\subsection{Determining Objective Function}
\label{sec:app:obj}

We now apply the blocking force model of the 1D PneuNet to our limb using the following additional assumptions.

\begin{enumerate}
    \setcounter{enumi}{4}
    \item $F(w)$ in eq. \eqref{eq:force} is the blocking force when only one side of the limb is actuated.
    \item The central axis of the limb is inextensible, and the location of the effective pressure is in the middle of the chamber's face, i.e., $t_c = \frac{h}{2} + t + \frac{w\tan(\pi/6)}{2}$ (See Fig. \ref{fig:limb_dimension}d).
    \item $\alpha$ is an unknown positive constant, independent of $\boldtheta$. This allows modeling of $F(\boldtheta)$ to be closed form.
\end{enumerate}

In the upright blocking pose, $\delta(z) \approx 0 \ \forall z \in [0, L]$ (Sec. \ref{subsec:prob_setup}).
Therefore eqn. (\ref{eq:force}) applies, becoming $F(w) = 2(\frac{7}{4}-\alpha) \bar{u} \frac{(m-1) hw(\frac{h}{2}+t+\frac{w}{2}\tan{(\frac{\pi}{6})})}{\bar{L}}$. We simplify the force model further to use as an objective function. Since $\boldtheta=w$, we omit the constant term, which only scales the variable terms, for simplicity (the constant = $\frac{2(7/4-\alpha)\bar{u}(m-1)h}{\bar{L}}$). The objective function becomes proportional to:

\begin{equation}
    \label{eq:obj_final}
    F(w) \propto w(\frac{h}{2}+t+\frac{w}{2}\tan (\frac{\pi}{6})).
\end{equation}

\noindent We emphasize that $F(w)$ is in a positive quadratic form, from which a monotonicity argument can be made under certain conditions. Therefore, the model does not need to be calibrated: we do not need specific values of the constants (e.g., $\alpha$) from eqn. \ref{eq:force} in order for Theorem 1 to be applied.

Rewriting the optimization problem using $w$,
\begin{align}
w^* &= \argmax_{w} F(w) \label{eq:obj2} \\
\text{s.t.}\hspace{0.3cm}
& V(w) \leq \bar{V} \label{eq:volume_const2} \\
& L_{max}(w) \geq \bar{L} \label{eq:gh_const2}
\end{align}

As shown in Fig. \ref{fig:limb_dimension}, we design our limb to satisfy the conditions for Assumption \ref{assm:sym_periodic}, i.e., the limb's cross-sectional geometry is symmetrical and periodic along the central axis, and the cross-section with the lowest radius of gyration (Fig. \ref{fig:limb_dimension}(c)) stays close to $z=0$. By applying Assumption \ref{assm:uniform_beam}, we use equation \eqref{eq:approx_greenhill} for $L_{max}$.

\subsection{Analytical Solution to Optimization Via Theorem \ref{thm:active}}
\label{sec:apply_numbers}

Combining all these observations, we find $w^*$ analytically by first showing that Assumptions \ref{assm:closed_bounded} and \ref{assm:f_increase} hold for any $w \in \Omega$, and then applying Theorem \ref{thm:active}. 
Applying the example design requirements and material properties from Table \ref{tab:variables} to our optimization problem gives the functions \eqref{eq:obj2}, \eqref{eq:volume_const2}, and \eqref{eq:gh_const2} in terms of only the unknown $w$.
These functions with their constants filled are omitted for brevity.

We now show that Assumption \ref{assm:closed_bounded} holds by analytically finding $\Omega$.
For $V(w)$ with the Table \ref{tab:variables} values, eqn. \eqref{eq:volume_const2} is smooth and strictly increasing for $w > 0$. 
Let $w_b \in \Omega$ be a value such that $V(w_b) = \bar{V}$. 
Because $\frac{dV(w)}{dw} > 0$, $w_b > w$ for any $w \in \{ w\ |\ V(w) < \bar{V}\}$, which means $w_b$ is an upper bound of $\Omega$. 
We also find that $L_{max}(w)$ (Eq. \eqref{eq:approx_greenhill}) is smooth and strictly increasing for any $w > 0$. Let $w_a \in \Omega$ be a value such that $L_{max}(w_a)=\bar{L}$. Because $\frac{dL_{max}(w)}{dw} > 0$, $w_a < w$ for any $w \in \{ w\ |\ L_{max}(w) > \bar{L}\}$, which means $w_a$ is a lower bound of $\Omega$. 
So, $\Omega = \{ w \ | \ w_a \leq w \leq w_b \}$. Using values from Table \ref{tab:variables}, we back-calculate the values of chamber widths $w_a$ and $w_b$ through $V(w)$ and $L_{max}(w)$ from Eq. \eqref{eq:approx_greenhill}, which results in $\Omega = [0.0367, 0.0452]$ meters.
Because, for any $\gamma \in (0,1)$, $\gamma w_a + (1-\gamma)w_b \in \Omega$, $\Omega$ is convex, and therefore, Assumption \ref{assm:closed_bounded} holds.

To show that Assumption \ref{assm:f_increase} holds, consider equation \eqref{eq:obj_final}. We find that $\frac{dF(w)}{dw} > 0$ on $w \in \Omega$, which shows that $F$ is strictly increasing (monotonic). Therefore, Assumption \ref{assm:f_increase} holds. Since Assumptions \ref{assm:closed_bounded} and \ref{assm:f_increase} hold, we apply Theorem \ref{thm:active} and find $w^*$ on the boundary of $\Omega$. Since we know that $F(w)$ is strictly increasing for $w\in \Omega$, $w^* = w_b$, which, is equal to $0.0452$ meters. We then rounded $w^*=0.045$ meters due to manufacturing uncertainty.

\section{Hardware Experimentation}
\label{sec:ex}

We validate our design optimization approach experimentally in hardware by manufacturing three different soft robotic limbs with varying design parameters ($w$): one of which is optimal ($w^*=0.045$ meters), one is sub-optimal ($w = 0.040$ meters), and another does not satisfy the anti-buckling constraint in Eq. \eqref{eq:gh_const} ($w = 0.030 \text{ meters } \notin \Omega$). We denote these limbs as $w_{45}$, $w_{40}$, and $w_{30}$, respectively. 

\begin{figure}[h]
    \centering
    \includegraphics[width=1\linewidth]{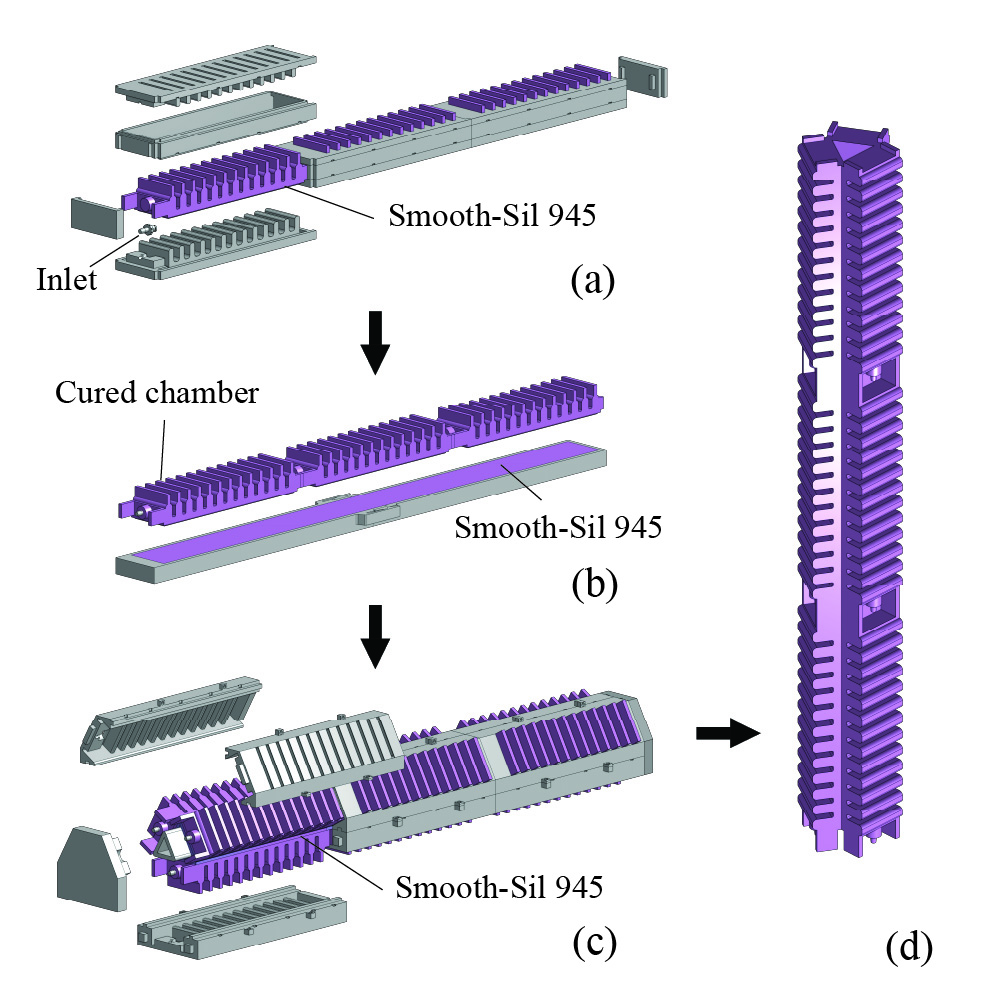}
    \caption{Manufacturing steps. (a) Mold casting chambers. (b) Overmolding seal for PneuNet. (c) Overmolding connection for multiple PneuNets to become one manipulator limb. (d) Final limb after removing supports.}
    \vspace{-0.3cm}
    \label{fig:manufact}
\end{figure}

\subsection{Fabrication}
\label{sec:ex:fab}

We fabricated the limb by performing a series of silicone mold casting \cite{rus2015design}, similarly to prior work \cite{ilievski2011soft,dickson2025safe}, which is adapted from the soft lithography technique \cite{xia1998soft}. 
The molds are 3D printed using ABS plastic and BambuLab X1C printers. 
The printer has a print bed of size $0.256 \times 0.256$ m$^2$, which is too small for molds for a limb of length $\bar{L} = 0.49$ m. 
We solve this issue by dividing the limb into three segments as shown in Fig. \ref{fig:manufact}(a). The segments are concatenated in series during each mold casting step using repeated sets of molds with connecting holes (See Fig. \ref{fig:manufact}(a)-(c)). This reduces failure over the connection area. 
We also ensure that the ratio $I/A$ of the connection area (Fig \ref{fig:limb_dimension}(e)) is higher than the minimum of the limb (Fig. \ref{fig:limb_dimension}(c)), so that Assumption \ref{assm:sym_periodic} still applies.
Our manufacturing technique includes three consecutive mold casting processes: casting chamber cavity (Fig. \ref{fig:manufact}(a)), overmolding the chamber's seal (this process creates the 1D PneuNet in Fig. \ref{fig:manufact}(b)), and attaching three sides of PneuNets (Fig. \ref{fig:manufact}(c)). 
The finalized limb is shown in Fig. \ref{fig:manufact}(d).

\subsection{Determining Maximum Stable Length, $L_{max}$}
\label{sec:ex:buck}

We evaluate Assumption \ref{assm:uniform_beam} by experimentally measuring the maximum stable height of each limb.
To do so, we manufactured each limb longer than its $L_{max}$, at roughly $0.6 \ m$.
We then vary the effective height $L$ of the limb and measure the tip deflection $\delta(L)$, denoted $\delta$ for simplicity, at the static equilibrium.
To imitate various $L$ of limbs, we design varying sizes of chamber clamps (Fig. \ref{fig:buckling_setup}(a)). Each clamp locks a certain number of chambers and has a length of multiples of $\ell$. As a result, the movable height of limbs, $L$, changes as shown in Fig. \ref{fig:buckling_setup}(a)-(c).
The limb and chamber clamps are attached to a stable base composed of 3D-printed parts and 80/20 aluminum bars.

\begin{figure}[h]
    \centering
    \includegraphics[width=1\linewidth]{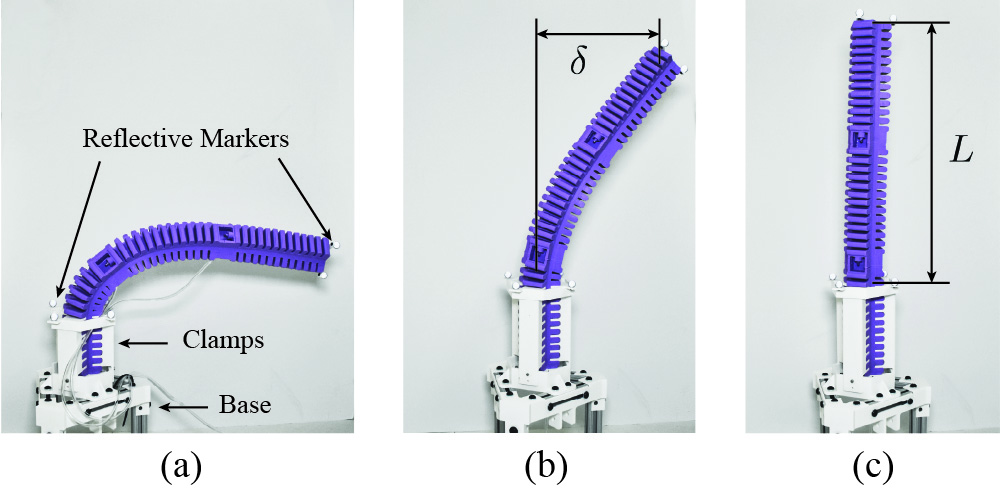}
    \caption{Buckling test setup: a) Limb buckled at maximum length, b) Limb deflected by $\delta$ due to self-buckling with lower-length, c) Limb unbuckled at length $L$.}
    \label{fig:buckling_setup}
\end{figure}

For each limb at each $L$, we measure $\delta$ (Fig. \ref{fig:buckling_setup}(b)) using a motion capture system (OptiTrack), which comprises 44 infrared cameras running at 120 hertz that detect the reflective markers (Fig. \ref{fig:buckling_setup}(a)). 
We define two sets of markers, which represent the poses of the central axis at the base ($z=0$) and at the tip ($z=L$). We record $\delta$ once the limb at each height $L$ has stabilized. We repeat the test five times and remove outliers, which are all test data $\delta_i$ such that $\delta_i - \tilde{\delta} > 0.05$ meters where $\tilde{\delta}$ is the median of that deflection for the same limb at length $L$. We then take the average of the filtered data.
Fig. \ref{fig:buckling_setup}(c) shows the limb at an unbuckled height, i.e., $L \leq L_{max}$. 
It is expected that $w_{30}$ limb buckles at all lengths.

\subsection{Optimal Design Validation}

We then measure the blocking force of each limb at the length $L=\bar{L}$ to validate the optimality of our result (See Fig. \ref{fig:force_setup}). 
The base of each limb is clamped so that only $0.49$ meters of the limb is free to move, as in Sec. \ref{sec:ex:buck}. 
We note that each limb satisfies the limited volume constraint \eqref{eq:volume_const2} at length $L=\bar{L}$.
Additionally, we connect a pneumatic system to one side of the limb to actuate and collect force readings at different internal pressures. 
The system is adapted from prior work \cite{dickson2025safe}, which consists of a pneumatic pump connected to an Arduino Mega microcontroller. 
We perform closed-loop pressure control per \cite{dickson2025safe}, measuring the internal pressure with a Honeywell MPRLS sensor, and the blocking force using a strain gauge sensor (Fig \ref{fig:force_setup}(a)). We record $F$ at $u \in[0, 290]$ hPa gauge pressure, only in an inflating direction. 
It is expected that the limb $w_{45}$ will result in the highest $F$ at all $u$.

\begin{figure}[h]
    \centering
    \includegraphics[width=1\linewidth]{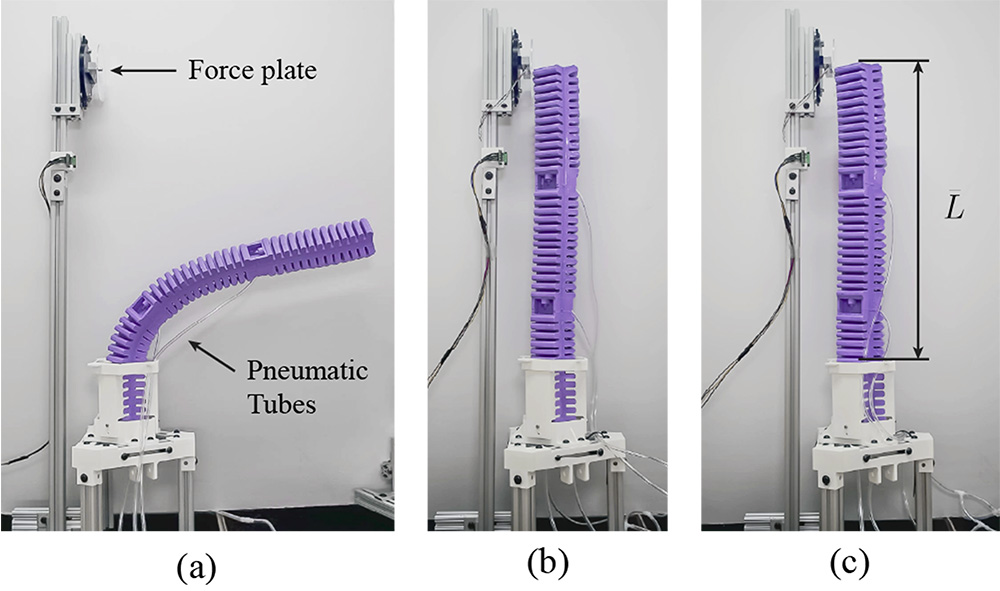}
    \caption{Blocking force test setup of limb, with (a) $w_{30}$, (b) $w_{40}$, (c) $w_{45}$.}
    \label{fig:force_setup}
\end{figure}

\section{Experimental Results}
\label{sec:results}

Results show that that our method can be used to distinguish between the three prototype designs.
Assumption \ref{assm:uniform_beam} to approximate $L_{max}$ holds with a maximum error on the order of a few centimeters, and the optimal $w_{45}$ manipulator produces the largest blocking force.

\subsection{Buckling Length}
\label{sec:res:buck}

To determine the buckling length, and therefore the quality of the Assumption \ref{assm:uniform_beam} eqn. \eqref{eq:approx_greenhill}, we plot the tip deflection $\delta$ at different lengths $L$ of each limb (Fig. \ref{fig:buckling_result}). 
Buckling instability is not immediately obvious at some specific $L$, but rather, the beam gradually deflects, which is consistent with the behavior of a linearly elastic column with initial imperfection studied in \cite{virgin2004postbuckling}.
To overcome this limitation, we propose an \textit{empirical} buckling length $\hat{L}_{max}$ which can be calculated from our data:

\begin{definition}
\label{def:Lmax}
    The empirical buckling length of the soft limb, $\hat{L}_{max}$, is the length at which the limb's rate of change in deflection along the change of length is maximum, i.e., $\hat{L}_{max} = \argmax_L \frac{d \delta}{dL}$.
\end{definition}

We apply Definition \ref{def:Lmax} by choosing the midpoint in the range that contains the steepest slope, which leads to the uncertainty of half the distance between each point, i.e., $\pm 0.007$ meters. 
We visualize the slope through piecewise linear interpolation between each pair of consecutive data points and $\hat{L}_{max}$ in cross symbols in Fig. \ref{fig:buckling_result}.
Table \ref{tab:buck_res} shows $L_{max}$ obtained through approximation \eqref{eq:approx_greenhill}, $\hat{L}_{max}$, and the discrepancy $\Delta L_{max} =|L_{max}-\hat{L}_{max}|$ for the three limbs.
The maximum error is 3.5 cm for these half-meter prototypes, approximately 6\% of the robot's height.

\begin{figure}[t]
    \centering
    \vspace{0.2cm}
    \includegraphics[width=1\linewidth]{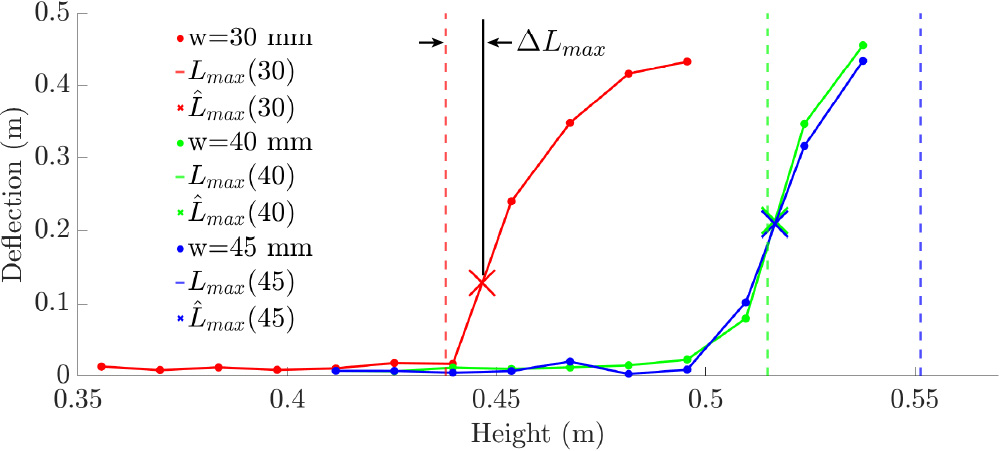}
    \caption{Empirical $\delta$ at varying $L$ of each limb where the cross shows empirical $\hat{L}_{max}$ and the dashed lines represent ${L}_{max}$ calculated from equation \eqref{eq:approx_greenhill}. $\Delta L_{max}$ is visualized for $w_{30}$ as an example.}
    \label{fig:buckling_result}
\end{figure}

\begin{table}[h]
    \centering
    \caption{Approximated and empirical buckling length of each limb}
    \begin{tabular}{c c c c}
    \hline \vspace{-0.2cm} \\
        \vspace{0.08cm}
        Limb & $w_{30}$ & $w_{40}$ & $w_{45}$ \\
        \hline \vspace{-0.2cm} \\
        \vspace{0.08cm}
        $L_{max} \ \text{m}$ & $0.438$ & $0.515$ & $0.551$ \\
        $\hat{L}_{max} \ (\pm 0.007 \ \text{m})$ & $0.447$ & $0.517$ & $0.517$ \\
        $\Delta L_{max} \ (\text{m})$ & $0.009$ & $0.002$ & $0.035$ \vspace{0.08cm} \\
        \hline
    \end{tabular}
    \label{tab:buck_res}
\end{table}

\subsection{Blocking Force}

We evaluate blocking force by sweeping over a range of pressure inputs for each limb.
To compare limbs, we shift the data so that each limb starts exerting force at the same normalized pressure $\Delta u = 0$.
We define $\Delta u = u - \bar{u}_w$ where $\bar{u}_w$ is a reference pressure such that the $F_w(\bar{u}_w) \geq \epsilon$. 
We choose $\epsilon = 9.81e^{-05}$ N, the uncertainty due to the force sensor. 
The offsets $\bar{u}_{40}$ and $\bar{u}_{45}$ are directly calculated from data, whereas the limb $w_{30}$ exerts $F_{30}(u)=0 \;\forall u \in [0, 290]$ hPa (See Fig. \ref{fig:force_setup}a), so we arbitrarily set $\bar{u}_{30} = \bar{u}_{40}$.

If our hypothesis is true that $w^*=w_{45}$, it should be that $F_{45}(\Delta u) > F(\Delta u)$ for the other designs.
To evaluate if this holds, we calculate the pressure at which a consistent relationship holds between the prototypes: $\Delta u_{crit}$ is $\argmax_{\Delta u} \Delta u$ such that $F_{45}(\Delta u) \leq F_{40}(\Delta u)$. 
In our case, $\Delta u_{crit} = 20.22$ hPa.
The blocking force measurements at different pressures in Fig. \ref{fig:force_result} show that, in the range $\Delta u>\Delta u_{crit}$, $F_{45}(\Delta u) > F_{40}(\Delta u) > F_{30}(\Delta u)$.

\begin{figure}[h]
    \centering
    \includegraphics[width=1\linewidth]{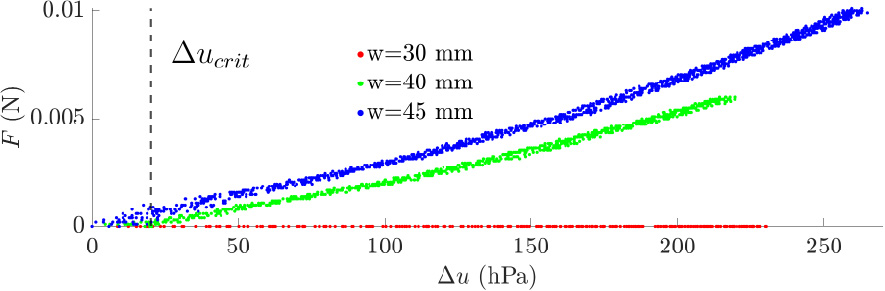}
    \caption{The data demonstrates force optimality of $w_{45}$ limb since all $F_{45}(\Delta u) > F_{40}(\Delta u) > F_{30}(\Delta u)$ for all $\Delta u > \Delta u_{crit}$.}
    \label{fig:force_result}
\end{figure}

\section{Discussion}
\label{sec:discuss}

Our hardware results confirm that Assumptions \ref{assm:sym_periodic} and \ref{assm:uniform_beam} apply to our limb designs, so Eq. \eqref{eq:approx_greenhill} can be used to predict $L_{max}$ for each of our limbs, and that the force-optimal design is $w^*=w_{45}$ as calculated.

\subsection{Results Interpretation}

Both the results from the buckling test and the blocking force test show agreement of Assumption \ref{assm:uniform_beam} and Theorem \ref{thm:active}, respectively. 
Errors in $L_{max}$ of at most 6\% are small in comparison to manufacturing differences and other unaccounted phenomena, so we conclude that eqn. \eqref{eq:approx_greenhill} can predict maximum stable robot height for a first-pass design.

Based on the blocking force data shown in Fig. \ref{fig:force_result}, $w_{45}$ produces the highest force after an initial pressure difference of approximately $\Delta u_{crit} \approxeq 20$ hPa.
Since the relationship is visually distinct for $\Delta u > \Delta u_{crit}$, we attribute this initial low-pressure behavior to noise in our test setup (vibrations or pressure regulation), and conclude that the result is valid particularly at higher pressures.
So, we conclude that the limb $w_{45}$ exerts the most force and that our methodology leads to an optimal design of this soft limb.

\subsection{Sources of Error}

Discrepancies exist between predicted versus observed buckling test $L_{max}$ for different limbs.
In particular, the errors for $w_{45}$ are larger than the other two prototypes.
Errors may be caused by manufacturing defects in $w_{45}$, for example, the top of PneuNet chamber walls may be connected to the sealing silicone during the second molding process in Fig. \ref{fig:manufact}(b). 
With the walls closer to the center of the cross-section, the average $I$ is smaller, and in turn, the experimental $L_{max}$ is smaller. 

In the blocking force hardware testing, each limb exhibits vibration while being inflated, which causes the limb to touch the force plate at varying forces at the same input pressure. 
The effect can be observed in Fig. \ref{fig:force_result} through the deviation along the $F$ axis at each $\Delta u$. 

\subsection{Limitations}

Limitations arise from the problem setup and its assumptions.
The problem setup assumes that there exists an analytical expression for $F(\boldtheta)$, which does not apply to all architectures of soft limbs. The blocking force expression for our limb is also assumed to follow the trend of the 1D PneuNet actuator by \cite{zhu2021modeling}. 
In a real-world application, meaningful interaction of a soft limb with the environment also does not merely occur at the blocking position \cite{cianchetti2018biomedical}, so it remains to be determined whether blocking force is truly a surrogate for force over the full robot workspace.

Additionally, our assumptions of known design requirements for proof-of-concept may not apply in all settings, and may not fully capture all goals.
Though the requirement of $L=\bar{L}$ physically represents the maximum workspace of the limb, it does not account for workspace under input constraints: an alternative problem formulation may choose to maximize some deflection at the same time as applied force.
An accurate analytical deflection model for this specific architecture of soft limb is out of scope, though future work could consider modified Euler-Bernoulli or Timoshenko bending \cite{olson_eulerbernoulli_2020,olson_generalizable_2021}.
Other design variables, such as $h$, $t$, and $\ell$, can also be included in the optimization problem, which would turn this into a multivariable optimization. 
We choose to optimize only $w$ and fix other variables for manufacturability and simplicity of experiment.

\section{Conclusion And Future Work}

This work presents a method to determine the geometric dimensions of a soft robot limb so that it applies the highest end effector forces under gravitational loading.
Choosing the highest-force geometry of a limb may allow soft robot designers to determine if their concept is appropriate for human-scale robotic interaction.
The approach analytically solves an optimization problem, relying on the monotonicity of the robot's interaction force $F$ and convex bounds from manufacturability. 
By using the maximum length of Euler-Bernoulli's beam \cite{greenhill} and Assumption \ref{assm:uniform_beam}, we are able to ensure the limb does not buckle under gravity while maximizing $F$. 
To our knowledge, this is the first time that a force-optimal geometry of a large soft robot arm can be calculated analytically, which answers a significant question in soft robotics: can soft robot manipulators be both large and strong? 
Our method provides an answer for any given class of designs.

Our methodology is limited by the requirement of a relationship between $F$ and $\boldtheta$, which may not be available for certain classes of soft limbs. We also use blocking force $F$ as an objective function, which constrains our design consideration to interactions at the tip at $\delta = 0$. 
For future work, the tip force could be modeled at $\delta > 0$ to optimize the overall force interaction with the environment. 
If the new problem setup breaks Assumptions \ref{assm:closed_bounded} and \ref{assm:f_increase}, we can explore a different optimization approach.
Additionally, future work can optimize multiple variables to reach a more optimal solution, or objective functions beyond force.
Finally, future work can assess the generalizability of this approach to other geometries of soft robot, as well as address the underlying question of practicality for human interaction.

\bibliographystyle{ieeetr}
\bibliography{references}
\end{document}